\documentclass[journal]{IEEEtran}
\IEEEoverridecommandlockouts

\usepackage[T1]{fontenc}
\usepackage[utf8]{inputenc}
\usepackage{amsmath,amssymb,bm}
\usepackage{graphicx}
\usepackage{booktabs,tabularx,array}
\usepackage{xcolor}
\usepackage{colortbl}
\usepackage{stfloats}
\usepackage{cite}
\usepackage{url}
\usepackage{hyperref}
\hypersetup{colorlinks=true, linkcolor=blue, citecolor=blue, urlcolor=blue}
\usepackage{tikz}
\usetikzlibrary{arrows.meta,positioning}

\DeclareMathOperator*{\argmax}{arg\,max}

\newcommand{\BKL}{\mathrm{BKL}}
\newcommand{\kl}{\mathrm{KL}}
\newcommand{\roll}{\operatorname{roll}}
\newcommand{\etal}{\textit{et al.}}

\title{\fontsize{21pt}{25pt}\selectfont PROBE-X: Learning-Free Cross-FOV Place Recognition}
\author{Jinseop~Lee%
\thanks{E-mail: jinseop.llee@gmail.com.}%
}
\begin{document}
\maketitle

\begin{abstract}
Under field-of-view (FOV) mismatch, pooling LiDAR features over unequal
angular support can distort compact retrieval keys and exclude correct
matches before geometric verification. We present PROBE-X, a learning-free method for single-scan
cross-FOV place recognition. Building on PROBE's probabilistic occupancy
representation, which models translation uncertainty, PROBE-X constructs
ring-mean retrieval keys using angular
masks conditioned on up to four fixed hypotheses for the unknown relative
heading. Candidate lists are merged by maximum cosine similarity. Each
retained candidate is aligned under the hypotheses that retrieved it and
assigned a geometric score over the angular overlap at the refined heading.
Across controlled-FOV and heterogeneous-LiDAR experiments, PROBE-X
achieves higher Recall@1 and PR-AUC than the evaluated single-scan
baselines in most settings with limited-FOV queries and panoramic
database scans.
On the controlled asymmetric-FOV sweep, single-scan PROBE-X also
outperforms the baselines evaluated with up to ten accumulated scans on
both metrics.
In two ablation settings, removing the angular masks from retrieval-key
construction reduces candidate coverage and Recall@1, even when
alignment and scoring still account for angular support.
Project page: \url{https://sites.google.com/view/probex-pr}.
\end{abstract}

\IEEEpeerreviewmaketitle

\section{Introduction}
\label{sec:intro}
\label{sec:introduction}

A robot may need to localize in a map built with a different LiDAR.
For example, a map acquired by a survey vehicle with a panoramic LiDAR
may later be queried by a robot with a directional solid-state sensor.
The robot's heading further determines which part of the mapped
structure it observes (Fig.~\ref{fig:teaser}).
We refer to this task as \emph{cross-FOV place recognition}: retrieving
a matching database scan despite differences in horizontal field of
view (FOV) and unknown relative heading. Limited FOV is a recognized source of low
overlap in global LiDAR localization~\cite{yin2024global}.

A common strategy in LiDAR place recognition is to retrieve candidates
using compact keys before aligning and comparing their full descriptors.
Scan Context (SC)~\cite{kim2018scan} and the polar-context formulation of
Scan Context++ (SC++)~\cite{kim2021scan} follow this approach. Under FOV mismatch,
pooling over unequal angular support can distort retrieval keys even for
scans of the same place. A correct match excluded during candidate
retrieval cannot be recovered by subsequent alignment. PROBE's restricted-FOV
experiments~\cite{lee2026probe} illustrate this limitation: brute-force
matching recovers matches missed by compact-key retrieval under asymmetric
FOV.

\begin{figure}[!t]
  \centering
  \includegraphics[width=\columnwidth]{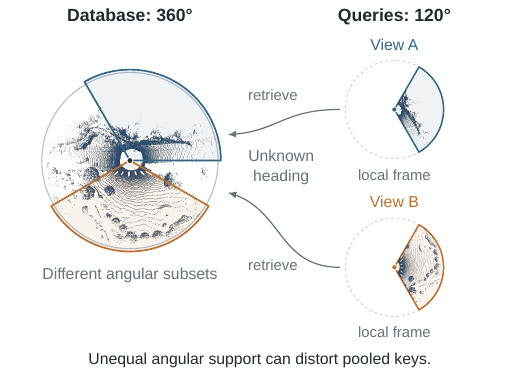}
  \caption{Two $120^\circ$ query views at the same location cover different
  angular subsets of a panoramic database scan (matching colors), with
  relative headings unknown during retrieval. The views are schematic
  rotations and crops of one point-cloud image. Regions outside each
  query wedge are unobserved.}
  \label{fig:teaser}
\end{figure}

\begin{figure*}[!t]
  \centering
  \includegraphics[width=\textwidth]{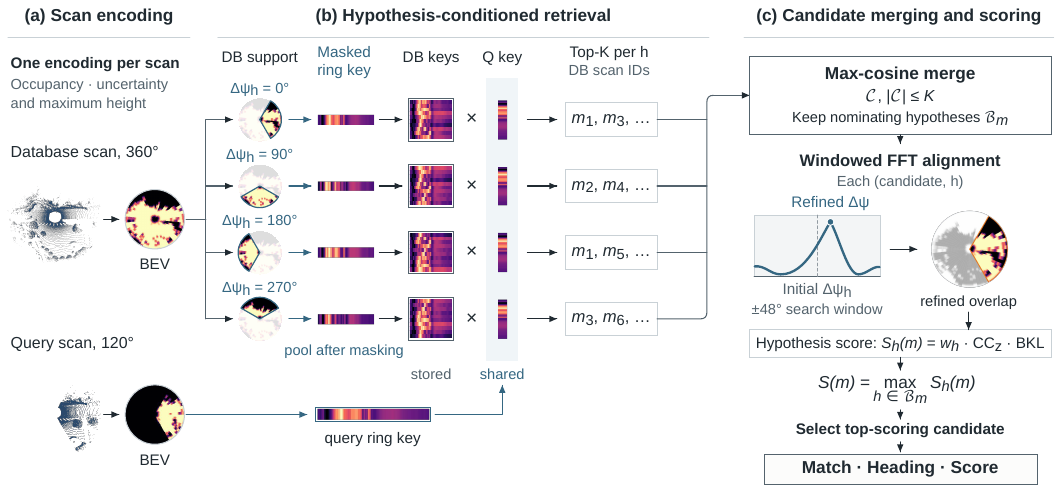}
  \caption{Schematic PROBE-X pipeline for a $360^\circ$ database and a $120^\circ$
  query. \textbf{(a)} Each scan is encoded once. \textbf{(b)} All four
  heading hypotheses survive the overlap gate and define database masks
  before ring-key pooling. Stored key banks are searched by cosine
  similarity using a shared query key.
  \textbf{(c)} Maximum-cosine merging retains at most $K$ distinct
  candidates and their nominating hypotheses. Each candidate--hypothesis
  pair is aligned within $\pm48^\circ$ of $\Delta\psi_h$ (gray dashed line)
  and scored on overlap at the refined $\Delta\psi$ (blue dot), with $w_h$
  fixed at its value for the initial hypothesis. Scores are maximized over each
  candidate's nominating hypotheses, then over candidates.}
  \label{fig:overview}
\end{figure*}

DVMM addresses horizontal FOV mismatch through azimuthal sliding during
candidate retrieval~\cite{duan2025dvmm}, while InLiER uses
rotation-invariant shortlisting with query-normalized histogram
intersection to accommodate partial queries~\cite{stathoulopoulos2026inlier}.
Because compatible angular support depends on relative heading, we
investigate how to construct compact ring keys when that heading is
unknown.

PROBE-X extends PROBE to cross-FOV retrieval using its probabilistic
occupancy representation and masked ring-mean keys under four fixed
relative-heading hypotheses spaced $90^\circ$ apart. It discards hypotheses
with insufficient angular overlap and merges candidates by maximum cosine
similarity. Each candidate is aligned and scored only under its nominating
hypotheses, with scoring restricted to the angular overlap at the refined heading
(Fig.~\ref{fig:overview}).

Our contributions are as follows:
\begin{enumerate}
  \item \textbf{Hypothesis-conditioned cross-FOV retrieval.} We apply
  heading-dependent masks before ring-key pooling and retain each
  candidate's nominating hypotheses for alignment and geometric scoring
  (Secs.~\ref{sec:branches}--\ref{sec:verification}).
  \item \textbf{Cross-FOV and cross-sensor evaluation.} We evaluate
  controlled FOV changes, bidirectional cross-sensor retrieval within
  and across sessions, and transfer to a legged platform under a common
  evaluator, including comparisons against baselines using accumulated
  scans (Sec.~\ref{sec:experiments}).
  \item \textbf{Controlled analysis of retrieval design.} Ablations in
  two settings show that conditioned geometric comparison cannot
  compensate for candidate losses from unconditioned keys within the
  resulting shortlists. Relative-heading
  analysis quantifies the gains and losses of multi-hypothesis retrieval
  (Sec.~\ref{sec:controls}).
\end{enumerate}

\section{Related Work}
\label{sec:related}

\textbf{Global scan descriptors.}
SC~\cite{kim2018scan} uses a polar grid and circular column shifts to
compare scans at different yaw angles; SC++~\cite{kim2021scan} adds a
complementary Cartesian representation.
M2DP~\cite{he2016m2dp} compresses multiview density projections using
singular value decomposition, whereas LiDAR Iris~\cite{wang2020lidar} encodes
height information in polar images and extracts binary signatures.
RING++ uses Radon and Fourier
transforms for roto-translation-invariant matching on sparse scan
maps~\cite{lu2023ring}. PROBE approximates occupancy probabilities under
translation uncertainty using adaptive polar blurring~\cite{lee2026probe}.

\textbf{Local features and submaps.}
STD~\cite{yuan2023std} uses triangles of structural keypoints, while
BTC~\cite{yuan2024btc} combines triangle geometry with local binary
descriptors. Both verify matches geometrically. MapClosures matches ORB
features in bird's-eye-view (BEV) density images of accumulated local
maps~\cite{gupta2026mapclosures}.
The published evaluations of DVMM and InLiER also use inputs formed by
accumulating ten consecutive scans~\cite{duan2025dvmm,stathoulopoulos2026inlier}.

\textbf{Learned place recognition.}
OverlapNet predicts scan overlap and relative yaw for loop
closure~\cite{chen2021overlapnet}. HeLiOS learns global descriptors for
heterogeneous LiDARs using overlap-based data mining and metric-learning
losses~\cite{jung2025helios}. BEVPlace++ aggregates rotation-equivariant
BEV features into rotation-invariant global
descriptors~\cite{luo2025bevplace}. UniLGL combines spatial and intensity
BEV representations with supervision for viewpoint invariance across
FOV-limited and panoramic LiDARs~\cite{shen2026unilgl}.

\textbf{Limited-FOV and heterogeneous LiDAR.}
SOLiD uses elevation statistics to reweight range--elevation and
azimuth--elevation histograms for FOV-constrained place
recognition~\cite{kim2024solid}. DVMM slides the narrower scan's azimuthal
profile over the wider scan's circularly padded profile for coarse
retrieval, then uses the estimated yaw to align cross-section occupancy
images for verification~\cite{duan2025dvmm}.
InLiER tokenizes height-sliced keypoints. Its rotation-invariant
shortlisting uses pairwise height-ceiling pruning to handle vertical FOV
mismatch and query-normalized histogram intersection to score candidates. Circular bitmask alignment estimates coarse yaw,
followed by token-guided geometric verification and pose
estimation~\cite{stathoulopoulos2026inlier}.
PROBE-X applies heading-conditioned angular masks before ring-key
pooling (Sec.~\ref{sec:method}).

\section{Hypothesis-Conditioned Cross-FOV Retrieval}
\label{sec:method}

\subsection{Problem Setup and Scan Representation}
\label{sec:problem}
Given a query scan $Q$ and database scans $\{D_m\}_{m=1}^{M}$, retrieval
returns a top-ranked database scan, a similarity score, and a
relative-heading estimate $\Delta\psi$. We assume approximately
gravity-aligned scans with known, contiguous horizontal FOVs and unknown
relative heading. Each scan is voxel-downsampled, and
points within 1\,m of the sensor in the horizontal plane are excluded.
The remaining points are encoded as a polar BEV grid of $N_r$ range rings and $N_s$ azimuth
bins with three channels inherited from PROBE~\cite{lee2026probe}: the
occupancy probability $\mu(r,s)$, the corresponding Bernoulli standard deviation
$\sigma=\sqrt{\mu(1-\mu)}$,
and the per-bin maximum height $z_{\max}(r,s)$. PROBE's adaptive polar blur approximates isotropic translation
uncertainty with an angular bandwidth inversely proportional to range
and adjusted by ring occupancy density. For zero-based ring index $r$,
the angular bandwidth is
\begin{equation}
  \sigma_\theta(r)=
  \max\!\left(\frac{\sigma_T\sqrt{\rho(r)}}{r_c\,\Delta_\theta},\,0.1\right),
  \label{eq:blur}
\end{equation}
in bin units, where $\sigma_T$ is the translation-uncertainty scale in
meters, $R_{\max}$ is the maximum horizontal range,
$\Delta_r=R_{\max}/N_r$, $r_c=(r+\tfrac12)\Delta_r$,
$\rho(r)=\max(N_s^{-1}\sum_s o(r,s),0.01)$ for the raw binary occupancy
$o$, and $\Delta_\theta=2\pi/N_s$. A radial Gaussian standard deviation of
$\max(\mathrm{round}(\sigma_T/\Delta_r),1)$ bins completes the separable
approximation. Occupancy is blurred with circular azimuthal and zero-padded
radial boundaries; heights remain unblurred. Within each retrieval run, database scans share a
coverage mask $\mathcal F^d$, and queries share $\mathcal F^q$.
Each mask contains the azimuth bins whose centers lie within the sensor's
FOV; its cardinality counts bins. When applied to a BEV grid, a mask is
broadcast over range rings and channels. Headings are represented modulo $360^\circ$. A positive
$\Delta\psi$ denotes an increasing-azimuth shift of the database grid
into the query frame; the corresponding yaw rotation from query to
database has the opposite sign.

The masks below model \emph{angular coverage under a heading hypothesis},
not physical visibility.

\subsection{Heading Hypotheses and Angular Masks}
\label{sec:branches}
To account for the unknown relative heading, we enumerate four fixed
hypotheses $\Delta\psi_h\in\{0^\circ,90^\circ,180^\circ,270^\circ\}$,
each defining a retrieval branch. The query and database conditioning masks are
\begin{align}
  \mathcal O_q^h&=\roll(\mathcal F^{d},+\Delta\psi_h),\label{eq:overlap_q}\\
  \mathcal O_d^h&=\mathcal F^{d}\cap\roll(\mathcal F^{q},-\Delta\psi_h),\label{eq:overlap_d}
\end{align}
where $\roll(\cdot,\Delta\psi)$ circularly shifts a mask by
$\mathrm{round}(\Delta\psi N_s/360^\circ)$ bins, with positive shifts
toward increasing azimuth. We weight each hypothesis
by the angular overlap relative to the narrower FOV:
\begin{equation}
  w_h=\frac{\big|\mathcal F^{q}\cap\roll(\mathcal F^{d},\Delta\psi_h)\big|}
           {\min\!\big(|\mathcal F^{q}|,\,|\mathcal F^{d}|\big)+\epsilon},
  \label{eq:visibility}
\end{equation}
where $\epsilon=10^{-6}$ is a numerical safeguard.
We discard hypotheses with $w_h<w_{\min}$, leaving at most four.
If either FOV is panoramic, all four hypotheses are retained. The tested
$120^\circ$--$70^\circ$ pair, with both FOVs centered forward, retains only
the forward hypothesis. The gate uses discretized FOV masks at a
hypothesized heading, whereas the ground-truth overlap criterion uses
continuous arcs at the true relative heading (Sec.~\ref{sec:protocol}).

\subsection{Hypothesis-Conditioned Ring-Key Retrieval}
\label{sec:retrieval}
The retrieval key is PROBE's ring-mean key~\cite{lee2026probe}.
Ring $r$ contributes its mean occupancy and the mean of its per-bin
maximum heights,
\begin{equation}
  \bar\mu_r=\tfrac{1}{N_s}{\textstyle\sum_{s}}\mu(r,s),\;
  \bar z_r=\tfrac{1}{N_s}{\textstyle\sum_{s}}z_{\max}(r,s),
  \label{eq:ringfeat}
\end{equation}
and the vector
$[\bar z_0,\ldots,\bar z_{N_r-1},\bar\mu_0,\ldots,\bar\mu_{N_r-1}]^\top$
has dimension $D=2N_r$. Vectors with norm greater than $\epsilon$ are
L2-normalized; a degenerate query key nominates no candidates.
PROBE-X changes the support used to construct this key:
\begin{align}
  \bm k_q^{h}&=\operatorname{RingKey}\big(\mathrm{BEV}^{q}\odot\mathcal O_q^h\big),\nonumber\\
  \bm k_{d,m}^{h}&=\begin{cases}
    \operatorname{RingKey}\big(\mathrm{BEV}^{d}_{m}\big), & \mathcal O_d^h=\mathcal F^{d},\\
    \operatorname{RingKey}\big(\mathrm{BEV}^{d}_{m}\odot\mathcal O_d^h\big), & \text{otherwise.}
  \end{cases}
  \label{eq:masked_keys}
\end{align}
With a panoramic database and a narrower query, the query key is shared
while the database keys use hypothesis-dependent masks
(Fig.~\ref{fig:overview}). With a panoramic query and a narrower database,
the database keys are reused unmasked while the query keys use
hypothesis-dependent masks.
The query mask and the unmasked database-key case can retain blurred
occupancy outside the physical FOV. Alignment and scoring use
explicit FOV intersections. At insertion, we store the database keys as
rows of a matrix $\bm K_d^h$ for each hypothesis. For nondegenerate keys,
the product $\bm K_d^{h}\bm k_q^{h}$ gives their cosine similarities.

\textbf{Candidate merging.} Each retained hypothesis nominates up to $K$
admissible database scans (Sec.~\ref{sec:protocol}) with the highest
unweighted cosine similarities, excluding values below $0.1$.
We rank their union by each scan's highest cosine similarity among its
nominating hypotheses and retain at most $K$ distinct scans. For each retained scan $m$,
$\mathcal B_m$ contains all nominating hypotheses. An empty shortlist
returns no match. The place score is
\begin{equation}
  S(m)=\max_{h\in\mathcal B_m}S_h(m),
  \label{eq:final}
\end{equation}
with $S_h$ the hypothesis score of Sec.~\ref{sec:verification}. For
scores at least $\epsilon$, the returned heading is the refined estimate
from the maximizing hypothesis. Scores below $\epsilon$ are reported as
zero.

\subsection{Heading Alignment and Geometric Scoring}
\label{sec:verification}
\textbf{Windowed alignment.} For each retained candidate--hypothesis
pair $(m,h)$, we refine the heading by FFT cross-correlation of the
masked maximum-height rows. The dependence on $(m,h)$ is omitted below
where unambiguous. We restrict the discrete shift to a circular window
$\mathcal W_h$ centered at $\mathrm{round}(\Delta\psi_hN_s/360^\circ)$:
\begin{equation}
  \Delta s^{*}=\argmax_{\Delta s\in\mathcal W_h}\;
  \mathrm{IFFT}\Big[\textstyle\sum_r \hat Z_q^{(r)}\,\overline{\hat Z_d^{(r)}}\Big](\Delta s).
  \label{eq:fft}
\end{equation}
Here $\hat Z^{(r)}=\mathrm{FFT}_s\big((z_{\max}\odot\mathcal O^{\rm al})(r,\cdot)\big)$.
The alignment masks use the FOV intersection at hypothesis $h$ in each
frame: $\mathcal O^{\rm al}_q=\mathcal F^q\cap\mathcal O_q^h$ and
$\mathcal O^{\rm al}_d=\roll(\mathcal O^{\rm al}_q,-\Delta\psi_h)=\mathcal O_d^h$.
The window has a half-width of
$\max(5,\lfloor N_s/7\rfloor)$ bins ($48^\circ$ at $N_s{=}60$).
The refined heading is
$\Delta\psi=\Delta s^{*}\,360^\circ/N_s$.

\textbf{Two-factor geometric scoring.} After alignment the
scoring support $\mathcal O=\mathcal F^{q}\cap\roll(\mathcal F^{d},\Delta\psi)$
is recomputed at the refined heading $\Delta\psi$. All database channels below are circularly
shifted into the query frame by $\Delta s^*$. Let
$p=\operatorname{clip}_{[\epsilon,1-\epsilon]}
(\mu(1-\sigma)+\tfrac12\sigma)$ be PROBE's
shrinkage-regularized occupancy. Shrinkage moves intermediate occupancies
toward $0.5$; clipping keeps the Bernoulli divergences finite.
Let $W=p_qp_d$ be the product occupancy weight and
$\mathcal Z=\{(r,s):s\in\mathcal O,\ z_q(r,s)\ne0,\ z_d(r,s)\ne0\}$,
where $z_q,z_d$ denote the aligned maximum-height channels (empty cells
are stored as zero). The height factor is the
$W$-weighted Pearson correlation of the $W$-mean-centered heights over the
cells of $\mathcal Z$, with negative correlations clipped to zero,
\begin{equation}
  \mathrm{CC}_z=\max\!\Big(0,\;
  \frac{\sum_{\mathcal Z} W\,\tilde z_q\tilde z_d}{\|\tilde z_q\|_W\,\|\tilde z_d\|_W}\Big),
  \label{eq:ccz}
\end{equation}
where $\tilde z_i=z_i-\sum_{\mathcal Z}Wz_i/\sum_{\mathcal Z}W$
and $\|x\|_W=(\sum_{\mathcal Z}Wx^2)^{1/2}$. The subtracted mean is
set to zero when $\sum_{\mathcal Z}W\le\epsilon$.
Otherwise, mean centering removes additive height offsets for fixed
support and weights.
We set $\mathrm{CC}_z=0$ when $|\mathcal Z|<5$; otherwise, we set it to
$1$ when $\|\tilde z_q\|_W\|\tilde z_d\|_W\le\epsilon$.
We derive the occupancy factor from the cellwise symmetric Bernoulli divergence
$D_{\mathrm{sym}}(r,s)=\tfrac12\big(\kl(p_q\|p_d)+\kl(p_d\|p_q)\big)$,
with both occupancies evaluated at $(r,s)$. Here
$\kl(a\|b)=a\ln(a/b)+(1-a)\ln((1-a)/(1-b))$.
We aggregate this divergence over
$\mathcal U=\{(r,s):s\in\mathcal O,\ \mu_q(r,s)+\mu_d(r,s)>0.1\}$ using a
column-trimmed mean. For column $s$, let $n_c(s)$ count the cells in
$\mathcal U$ within that column and $d_c(s)$ denote their mean value of
$D_{\mathrm{sym}}(r,s)$.
Of the $N_{\mathrm{col}}$ nonempty columns, we discard the
$\mathrm{round}(\eta N_{\mathrm{col}})$ with the highest mean divergence.
The remaining columns form
$\mathcal S_{\mathrm{kept}}$, and
\begin{equation}
  \BKL=\exp\!\Big(-\frac{\sum_{s\in\mathcal S_{\mathrm{kept}}}n_c(s)\,d_c(s)}
                        {\sum_{s\in\mathcal S_{\mathrm{kept}}}n_c(s)}\Big)\in(0,1]
  \quad(\mathcal U\neq\emptyset).
  \label{eq:bkl}
\end{equation}
The hypothesis receives zero geometric score when $|\mathcal U|<5$.
On the same aligned grids, with full angular coverage and $\eta=0$,
\eqref{eq:bkl} has the form of PROBE's $\mathcal J_{KL}$, apart from
the occupancy-union threshold. The hypothesis score multiplies the two
geometric factors by the angular overlap weight evaluated at the initial
heading hypothesis in \eqref{eq:visibility},
\begin{equation}
  S_h(m)=w_h\;\mathrm{CC}_z\;\BKL.
  \label{eq:branch_score}
\end{equation}
PROBE-X retains PROBE's two-factor geometric structure but uses weighted
height correlation, conditioned support, and column trimming.

\subsection{Search Coverage, Complexity, and Parameters}
\label{sec:complexity}
The retained windows $\mathcal W_h$ ($\pm48^\circ$ at $N_s{=}60$)
jointly cover the full circle when all four hypotheses survive.
Each candidate is searched only within the windows of its nominating
hypotheses. Section~\ref{sec:controls} reports how often a query's positive
matches all have relative headings outside the retained window in the
tested narrow--narrow pairs.

Let $H\le4$ be the number of retained hypotheses and $\mathcal C$ the
merged candidate set, with $|\mathcal C|\le K$. Key comparison and candidate
selection cost $O(HMD+HM\log M+HK\log(HK))$, where $D=2N_r$.
Verification costs
$O((\sum_{m\in\mathcal C}|\mathcal B_m|)\,N_rN_s\log N_s)$,
bounded by $O(KH N_rN_s\log N_s)$, independently of $M$.
Unless explicitly varied in a controlled comparison, PROBE-X uses
the same parameters across datasets: voxel size $0.5$\,m, range
$1$--$100$\,m, $\sigma_T=2$\,m, $N_r\times N_s=50\times60$ (2\,m rings),
$K=20$ per hypothesis and in total, $w_{\min}=0.3$, $\eta=0.1$.

\section{Experimental Evaluation}
\label{sec:experiments}
\begin{table*}[!t]
  \centering
  \caption{FOV sweep on NCLT 2012-05-26 (single-session), reported as Recall@1 / AUC. Symmetric: DB $=$ Q $=F^\circ$; asymmetric: DB $360^\circ$, Q $=F^\circ$. Mean averages all eleven settings equally. $\dagger$: learned. $\ddagger$: up to ten accumulated scans, each cropped to the stated FOV before accumulation; excluded from ranking (Sec.~\ref{sec:protocol}). Bold/underline indicate best/second-best single-scan values per metric, ranked at three-decimal precision.}
  \label{tab:fov_sweep}
  \footnotesize\setlength{\tabcolsep}{5pt}
  \begin{tabular}{@{}lc@{\hspace{12pt}}ccccc@{\hspace{12pt}}ccccc@{\hspace{12pt}}c@{}}
    \toprule
    & & \multicolumn{5}{c}{Symmetric (DB $=$ Q $=F^\circ$)} & \multicolumn{5}{c}{Asymmetric (DB $360^\circ$, Q $=F^\circ$)} & \\
    \cmidrule(lr){3-7}\cmidrule(lr){8-12}
    Method & $360^\circ$ & $270^\circ$ & $180^\circ$ & $120^\circ$ & $90^\circ$ & $60^\circ$ & $270^\circ$ & $180^\circ$ & $120^\circ$ & $90^\circ$ & $60^\circ$ & Mean \\
    \midrule
    SC++~\cite{kim2021scan} & \underline{.91}/\underline{.70} & .56/.38 & .66/.22 & .76/.17 & .73/.12 & .47/.03 & .57/.36 & .08/.01 & .02/.00 & .00/.00 & .00/.00 & .43/.18 \\
    SOLiD~\cite{kim2024solid} & .59/.40 & .40/.21 & .41/.23 & .44/.25 & .39/.20 & .36/.07 & .41/.15 & .24/.03 & .08/.00 & .03/.00 & .02/.00 & .31/.14 \\
    PROBE~\cite{lee2026probe} & \textbf{.93}/\textbf{.84} & .44/.32 & \textbf{.68}/.57 & .81/\underline{.59} & \underline{.85}/.55 & \underline{.73}/.30 & \underline{.71}/\underline{.52} & .13/.01 & .02/.00 & .01/.00 & .01/.00 & .48/.34 \\
    DVMM~\cite{duan2025dvmm} & .48/.36 & .24/.16 & .38/.23 & .49/.23 & .41/.11 & .38/.08 & .50/.39 & .42/.21 & \underline{.32}/\underline{.09} & \underline{.25}/\underline{.04} & \underline{.19}/\underline{.02} & .37/.18 \\
    DVMM$^\ddagger$~\cite{duan2025dvmm} & .66/.63 & .26/.20 & .42/.39 & .45/.32 & .42/.25 & .36/.10 & .58/.52 & .34/.18 & .19/.05 & .10/.02 & .07/.01 & .35/.24 \\
    InLiER~\cite{stathoulopoulos2026inlier} & .77/.61 & \underline{.57}/\underline{.40} & .54/.35 & .53/.32 & .40/.22 & .28/.15 & .68/.50 & \underline{.44}/\underline{.28} & .15/.08 & .07/.04 & .02/.01 & .40/.27 \\
    InLiER$^\ddagger$~\cite{stathoulopoulos2026inlier} & .88/.79 & .73/.66 & .68/.56 & .80/.60 & .84/.50 & .72/.47 & .83/.74 & .79/.69 & .64/.55 & .39/.31 & .18/.13 & .68/.54 \\
    HeLiOS$^\dagger$~\cite{jung2025helios} & .83/.67 & .48/.35 & \underline{.67}/\textbf{.59} & \textbf{.85}/\textbf{.72} & \textbf{.88}/\textbf{.72} & \textbf{.84}/\textbf{.60} & .71/.46 & .37/.11 & .23/.04 & .17/.02 & .10/.00 & \underline{.56}/\underline{.39} \\
    PROBE-X & .90/\textbf{.84} & \textbf{.74}/\textbf{.60} & .65/\underline{.59} & \underline{.81}/\textbf{.73} & \underline{.85}/\underline{.72} & .69/\underline{.42} & \textbf{.89}/\textbf{.83} & \textbf{.83}/\textbf{.75} & \textbf{.72}/\textbf{.62} & \textbf{.61}/\textbf{.48} & \textbf{.38}/\textbf{.22} & \textbf{.73}/\textbf{.62} \\
    \bottomrule
  \end{tabular}
\end{table*}

\subsection{Datasets, Protocol, and Baselines}
\label{sec:protocol}
\begin{figure}[!t]
  \centering
  \includegraphics[width=\columnwidth]{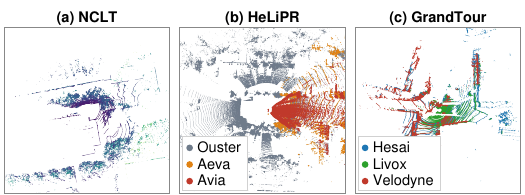}
  \caption{One top-down frame per dataset on a common $\pm35$\,m scale.
  \textbf{(a)} NCLT (Velodyne HDL-32E, height-colored). \textbf{(b)}
  HeLiPR: the $360^\circ$ Ouster OS2-128, $120^\circ$ Aeva Aeries-II, and
  $70^\circ$ Livox Avia overlaid in one calibrated frame, colored by sensor.
  \textbf{(c)} GrandTour: Hesai XT32, Livox Mid-360, and Velodyne VLP-16,
  all $360^\circ$, overlaid likewise.}
  \label{fig:datasets}
\end{figure}

\textbf{Datasets.} NCLT~\cite{carlevaris2016university} session 2012-05-26
provides a controlled FOV sweep through software cropping.
HeLiPR~\cite{minwoo2023helipr} provides cross-sensor pairs from three
asynchronous LiDARs on one vehicle: $360^\circ$ Ouster, $120^\circ$ Aeva,
and $70^\circ$ Avia. We evaluate KAIST and Riverside within session $05$
and across sessions $05\to06$. GrandTour~\cite{frey2026grandtour} tests
transfer to the quadruped ANYmal with three panoramic LiDARs in one
single-session and one multi-session setting. Sensor models and example
scans appear in Fig.~\ref{fig:datasets}.

\textbf{Protocol.} Database and query keyframes are selected at
displacements of at least $5$\,m from the preceding keyframe.
Single-session queries search only earlier database frames separated
from the query by at least $20$\,m of travel. For cross-sensor pairs,
this cutoff is measured backward along the database trajectory from
query time. Multi-session queries search the whole database.
All methods evaluate the same queries and admissible database sets,
excluding queries with no admissible frames. Let $A_Q$ and $A_{D_m}$
be the continuous FOV arcs of $Q$ and $D_m$ on a common azimuth circle,
oriented by their ground-truth headings, and $\ell(\cdot)$ denote
angular length. An admissible $D_m$ is positive for $Q$ if and only if
their horizontal distance is less than $7.5$\,m and
\begin{equation}
  \nu(Q,D_m)=\frac{\ell(A_Q\cap A_{D_m})}
    {\min\!\{\ell(A_Q),\ell(A_{D_m})\}}\ge0.25.
  \label{eq:gt_overlap}
\end{equation}
If either FOV is $360^\circ$, then $\nu(Q,D_m)=1$, so the angular
condition excludes no positives. Recall@1 is the fraction of evaluated
queries with positives whose top-1 is positive. We construct the
precision--recall curve by sweeping a threshold over the distinct scores
of fixed top-1 matches. Accepted non-positive
matches, including those from queries without positives, are false
positives. AUC is the trapezoidal area under this curve with an initial point at
(recall, precision) $=(0,1)$. Queries returning no match contribute no
prediction but remain in the recall denominator if they have positives.
A common evaluator processes top-1 outputs, and group means
weight cells equally.

All single-scan methods receive identical point clouds clipped to $100$\,m
horizontally, then apply their own range settings.
PROBE-X and PROBE use a $100$\,m, 50-ring grid instead of
PROBE's $80$\,m, 40-ring grid. The range setting and PROBE-X retrieval
design were selected using the reported benchmarks, without independent
held-out validation.

\textbf{Baselines.} We evaluate
SC++~\cite{kim2021scan}, SOLiD~\cite{kim2024solid},
PROBE~\cite{lee2026probe}, DVMM~\cite{duan2025dvmm},
InLiER~\cite{stathoulopoulos2026inlier}, and the learned
HeLiOS~\cite{jung2025helios} on single scans.
SC++ is included only in the FOV sweep (Table~\ref{tab:fov_sweep}).
SOLiD uses dataset-fixed encoding
parameters shared by database and query scans (maximum 3-D range:
$80$\,m on NCLT; $120$\,m otherwise), with cosine-similarity retrieval
over all admissible R-SOLiD descriptors. We use released HeLiOS
weights trained on HeLiPR sequences including the evaluated KAIST and
Riverside sessions. Tables~\ref{tab:fov_sweep}--\ref{tab:grandtour} also report
DVMM and InLiER ($\ddagger$) with each keyframe and up to nine preceding
raw scans aligned to the keyframe's coordinate frame using ground-truth poses.
For the main comparisons, DVMM and InLiER use their released range limits
of $50$\,m and $100$\,m, respectively, for both single-scan and
accumulated inputs. The accumulated-input rows
use the same queries and ground truth as the single-scan rows, but their
preprocessing differs from the published pipelines.

\subsection{Controlled FOV Sweep}
\label{sec:fov_results}
On NCLT session 2012-05-26, queries are cropped to
$F\in\{270,180,120,90,60\}^\circ$ and matched against the panoramic database
(asymmetric block). In the symmetric control block, the database and
query are cropped identically. The panoramic pair is the reference. The
angular ground-truth term can exclude positives in the symmetric block
but not the asymmetric block. Their positive sets can therefore differ,
so a direct comparison between blocks does not isolate FOV mismatch.

In the asymmetric block (Table~\ref{tab:fov_sweep}, right)
single-scan PROBE-X has the highest Recall@1 and AUC at all five widths,
over both the single-scan and the accumulated-input rows, with Recall@1
decreasing from $0.89$ at $270^\circ$ to $0.38$ at $60^\circ$.
SC++ and PROBE have Recall@1 at or below $0.13$ at query FOVs of
$180^\circ$ or narrower. SOLiD reaches $0.24$ at $180^\circ$ and remains
below $0.08$ at narrower widths. Mean Recall@1/AUC over the five widths
is $0.68/0.58$ for PROBE-X,
$0.56/0.48$ for InLiER$^\ddagger$, and $0.26/0.16$ for DVMM$^\ddagger$.
At query FOVs of $90^\circ$ and $60^\circ$, PROBE-X exceeds
InLiER$^\ddagger$ in Recall@1 by $0.22$ and $0.20$, respectively. Over the five symmetric widths, InLiER$^\ddagger$ has
similar mean Recall@1 to PROBE-X (both round to $0.75$), but lower mean
AUC ($0.56$ versus $0.61$).

Among the single-scan methods, the Recall@1 ranking differs in the
symmetric block: PROBE-X leads at
$270^\circ$, PROBE at $180^\circ$, and HeLiOS at
$120^\circ$--$60^\circ$. PROBE leads the panoramic reference. The
mean over all eleven cells is highest for PROBE-X
($0.73/0.62$), followed by HeLiOS ($0.56/0.39$) among single-scan
methods. Increasing the range from
$80$ to $100$\,m changes mean AUC over the five asymmetric cells by
$+0.001$ for PROBE-X and by less than $0.001$ in magnitude for PROBE.

\subsection{Cross-Sensor Retrieval on HeLiPR}
\label{sec:heterogeneous}
All nine HeLiPR KAIST sensor pairings are evaluated within session
$05$ and across sessions $05\to06$ (Table~\ref{tab:cross}), so both
retrieval directions appear, with same-sensor and narrow--narrow pairs
as controls.

PROBE-X achieves the highest AUC among the evaluated single-scan methods
in all twelve KAIST cross-sensor settings. Over the six off-diagonal
pairs, it reaches mean Recall@1/AUC of
$0.59/0.57$ within session and $0.53/0.51$ across sessions, against
$0.40/0.25$ and $0.42/0.33$ for InLiER.
At the settings reported in Table~\ref{tab:cross}, PROBE-X leads six of
the eight wide$\leftrightarrow$narrow cells in Recall@1; the exceptions
are the Avia-database, Ouster-query cells, where InLiER is marginally
ahead in both session settings. On the narrow--narrow Aeva$\leftrightarrow$Avia
pairs, PROBE-X leads in Recall@1 within session but falls behind HeLiOS
across sessions. On the same-sensor diagonal, PROBE-X leads in Recall@1 within
session; across sessions, InLiER leads for Ouster and PROBE for Aeva and
Avia.

With up to ten accumulated scans, InLiER$^\ddagger$ exceeds single-scan
PROBE-X in eleven of the twelve off-diagonal cells in Recall@1 (one tie)
and ten in AUC. Its mean Recall@1/AUC is
$0.67/0.60$ within session and $0.62/0.59$ across sessions.
DVMM$^\ddagger$ averages $0.05$ Recall@1 in both session
settings, below its corresponding single-scan means.

\begin{table*}[!t]
  \centering
  \caption{Cross-sensor place recognition on HeLiPR KAIST (Recall@1 / AUC). Rows denote database sensors and columns query sensors. $\dagger$: learned; training includes the evaluated KAIST sessions (Sec.~\ref{sec:protocol}). $\ddagger$: up to ten accumulated scans; excluded from ranking (Sec.~\ref{sec:protocol}). Bold/underline indicate best/second-best single-scan values per metric within each KAIST pairing, ranked at three-decimal precision.}
  \label{tab:cross}
  \scriptsize\setlength{\tabcolsep}{2pt}
  \begin{tabular*}{\textwidth}{@{\extracolsep{\fill}}c c c c@{}}
    \toprule
    DB$\backslash$Q & \textbf{Ouster} ($360^\circ$) & \textbf{Aeva} ($120^\circ$) & \textbf{Avia} ($70^\circ$) \\ \midrule
    \multicolumn{4}{l}{\textit{Single-session (KAIST05 database, KAIST05 query)}} \\ \midrule
    \textbf{Ouster} & \begin{tabular}[c]{@{}l@{}}SOLiD .65/.50 $\cdot$ PROBE \underline{.92}/\underline{.90} $\cdot$ DVMM .53/.52 \\ DVMM$^\ddagger$ .56/.56 $\cdot$ InLiER .90/.86 $\cdot$ InLiER$^\ddagger$ .93/.90 \\ HeLiOS$^\dagger$ .82/.70 $\cdot$ PROBE-X \textbf{.95}/\textbf{.93}\end{tabular} & \begin{tabular}[c]{@{}l@{}}SOLiD .04/.00 $\cdot$ PROBE .02/.00 $\cdot$ DVMM .24/.15 \\ DVMM$^\ddagger$ .03/.01 $\cdot$ InLiER \underline{.29}/\underline{.19} $\cdot$ InLiER$^\ddagger$ .69/.64 \\ HeLiOS$^\dagger$ .12/.02 $\cdot$ PROBE-X \textbf{.68}/\textbf{.66}\end{tabular} & \begin{tabular}[c]{@{}l@{}}SOLiD .04/.00 $\cdot$ PROBE .03/.00 $\cdot$ DVMM \underline{.41}/\underline{.25} \\ DVMM$^\ddagger$ .04/.02 $\cdot$ InLiER .35/.25 $\cdot$ InLiER$^\ddagger$ .56/.52 \\ HeLiOS$^\dagger$ .12/.01 $\cdot$ PROBE-X \textbf{.56}/\textbf{.53}\end{tabular} \\
    \arrayrulecolor{black!35}\cmidrule[0.3pt](lr){1-4}\arrayrulecolor{black}
    \textbf{Aeva} & \begin{tabular}[c]{@{}l@{}}SOLiD .01/.00 $\cdot$ PROBE .14/.04 $\cdot$ DVMM .18/.13 \\ DVMM$^\ddagger$ .05/.05 $\cdot$ InLiER \underline{.28}/\underline{.18} $\cdot$ InLiER$^\ddagger$ .60/.56 \\ HeLiOS$^\dagger$ .06/.00 $\cdot$ PROBE-X \textbf{.49}/\textbf{.48}\end{tabular} & \begin{tabular}[c]{@{}l@{}}SOLiD .77/.57 $\cdot$ PROBE \underline{.93}/\underline{.92} $\cdot$ DVMM .49/.24 \\ DVMM$^\ddagger$ .16/.14 $\cdot$ InLiER .83/.58 $\cdot$ InLiER$^\ddagger$ .94/.88 \\ HeLiOS$^\dagger$ .91/.86 $\cdot$ PROBE-X \textbf{.94}/\textbf{.93}\end{tabular} & \begin{tabular}[c]{@{}l@{}}SOLiD .27/.05 $\cdot$ PROBE .61/\underline{.53} $\cdot$ DVMM .43/.13 \\ DVMM$^\ddagger$ .13/.08 $\cdot$ InLiER .52/.30 $\cdot$ InLiER$^\ddagger$ .85/.71 \\ HeLiOS$^\dagger$ \underline{.72}/.43 $\cdot$ PROBE-X \textbf{.74}/\textbf{.70}\end{tabular} \\
    \arrayrulecolor{black!35}\cmidrule[0.3pt](lr){1-4}\arrayrulecolor{black}
    \textbf{Avia} & \begin{tabular}[c]{@{}l@{}}SOLiD .01/.00 $\cdot$ PROBE .12/.03 $\cdot$ DVMM .25/.24 \\ DVMM$^\ddagger$ .02/.02 $\cdot$ InLiER \textbf{.39}/\underline{.29} $\cdot$ InLiER$^\ddagger$ .49/.43 \\ HeLiOS$^\dagger$ .06/.00 $\cdot$ PROBE-X \underline{.39}/\textbf{.36}\end{tabular} & \begin{tabular}[c]{@{}l@{}}SOLiD .32/.05 $\cdot$ PROBE .51/\underline{.45} $\cdot$ DVMM .36/.31 \\ DVMM$^\ddagger$ .05/.04 $\cdot$ InLiER .58/.28 $\cdot$ InLiER$^\ddagger$ .83/.74 \\ HeLiOS$^\dagger$ \underline{.69}/.40 $\cdot$ PROBE-X \textbf{.70}/\textbf{.68}\end{tabular} & \begin{tabular}[c]{@{}l@{}}SOLiD .72/.45 $\cdot$ PROBE \underline{.92}/\underline{.87} $\cdot$ DVMM .62/.58 \\ DVMM$^\ddagger$ .70/.68 $\cdot$ InLiER .92/.58 $\cdot$ InLiER$^\ddagger$ .94/.90 \\ HeLiOS$^\dagger$ .88/.80 $\cdot$ PROBE-X \textbf{.96}/\textbf{.92}\end{tabular} \\
    \midrule
    \multicolumn{4}{l}{\textit{Multi-session (KAIST05 database, KAIST06 query)}} \\ \midrule
    \textbf{Ouster} & \begin{tabular}[c]{@{}l@{}}SOLiD .53/.37 $\cdot$ PROBE \underline{.85}/\underline{.85} $\cdot$ DVMM .79/.79 \\ DVMM$^\ddagger$ .82/.82 $\cdot$ InLiER \textbf{.93}/\textbf{.93} $\cdot$ InLiER$^\ddagger$ .96/.95 \\ HeLiOS$^\dagger$ .79/.69 $\cdot$ PROBE-X .85/.85\end{tabular} & \begin{tabular}[c]{@{}l@{}}SOLiD .04/.00 $\cdot$ PROBE .03/.01 $\cdot$ DVMM .35/.30 \\ DVMM$^\ddagger$ .04/.02 $\cdot$ InLiER \underline{.38}/\underline{.32} $\cdot$ InLiER$^\ddagger$ .68/.66 \\ HeLiOS$^\dagger$ .11/.01 $\cdot$ PROBE-X \textbf{.62}/\textbf{.61}\end{tabular} & \begin{tabular}[c]{@{}l@{}}SOLiD .05/.00 $\cdot$ PROBE .05/.01 $\cdot$ DVMM \underline{.56}/\underline{.48} \\ DVMM$^\ddagger$ .03/.01 $\cdot$ InLiER .52/.46 $\cdot$ InLiER$^\ddagger$ .66/.65 \\ HeLiOS$^\dagger$ .14/.02 $\cdot$ PROBE-X \textbf{.64}/\textbf{.63}\end{tabular} \\
    \arrayrulecolor{black!35}\cmidrule[0.3pt](lr){1-4}\arrayrulecolor{black}
    \textbf{Aeva} & \begin{tabular}[c]{@{}l@{}}SOLiD .01/.00 $\cdot$ PROBE .12/.03 $\cdot$ DVMM .11/.06 \\ DVMM$^\ddagger$ .04/.03 $\cdot$ InLiER \underline{.30}/\underline{.24} $\cdot$ InLiER$^\ddagger$ .54/.52 \\ HeLiOS$^\dagger$ .12/.01 $\cdot$ PROBE-X \textbf{.42}/\textbf{.40}\end{tabular} & \begin{tabular}[c]{@{}l@{}}SOLiD .57/.45 $\cdot$ PROBE \textbf{.89}/\textbf{.89} $\cdot$ DVMM .49/.37 \\ DVMM$^\ddagger$ .21/.19 $\cdot$ InLiER .74/.64 $\cdot$ InLiER$^\ddagger$ .91/.89 \\ HeLiOS$^\dagger$ \underline{.85}/.77 $\cdot$ PROBE-X .84/\underline{.84}\end{tabular} & \begin{tabular}[c]{@{}l@{}}SOLiD .23/.07 $\cdot$ PROBE .46/.42 $\cdot$ DVMM .43/.26 \\ DVMM$^\ddagger$ .10/.06 $\cdot$ InLiER .50/.34 $\cdot$ InLiER$^\ddagger$ .79/.75 \\ HeLiOS$^\dagger$ \textbf{.71}/\underline{.50} $\cdot$ PROBE-X \underline{.64}/\textbf{.62}\end{tabular} \\
    \arrayrulecolor{black!35}\cmidrule[0.3pt](lr){1-4}\arrayrulecolor{black}
    \textbf{Avia} & \begin{tabular}[c]{@{}l@{}}SOLiD .01/.00 $\cdot$ PROBE .09/.02 $\cdot$ DVMM .26/.24 \\ DVMM$^\ddagger$ .02/.01 $\cdot$ InLiER \textbf{.39}/\underline{.31} $\cdot$ InLiER$^\ddagger$ .41/.36 \\ HeLiOS$^\dagger$ .10/.01 $\cdot$ PROBE-X \underline{.37}/\textbf{.33}\end{tabular} & \begin{tabular}[c]{@{}l@{}}SOLiD .17/.04 $\cdot$ PROBE .26/.22 $\cdot$ DVMM .35/\underline{.33} \\ DVMM$^\ddagger$ .06/.05 $\cdot$ InLiER .43/.30 $\cdot$ InLiER$^\ddagger$ .67/.62 \\ HeLiOS$^\dagger$ \textbf{.54}/.32 $\cdot$ PROBE-X \underline{.46}/\textbf{.44}\end{tabular} & \begin{tabular}[c]{@{}l@{}}SOLiD .50/.38 $\cdot$ PROBE \textbf{.91}/\textbf{.90} $\cdot$ DVMM .66/.65 \\ DVMM$^\ddagger$ .75/.75 $\cdot$ InLiER .89/.83 $\cdot$ InLiER$^\ddagger$ .93/.91 \\ HeLiOS$^\dagger$ .89/.84 $\cdot$ PROBE-X \underline{.91}/\underline{.89}\end{tabular} \\
    \bottomrule
  \end{tabular*}
\end{table*}

Across the twelve KAIST off-diagonal cells, increasing the range from
$80$ to $100$\,m raises mean AUC by $0.092$ for PROBE-X and $0.030$ for
PROBE. DVMM's mean AUC over the same cells is $0.240$, $0.348$, and
$0.342$ at range limits of $50$, $80$, and $100$\,m, respectively.
Among single-scan methods, PROBE-X retains the highest AUC in every cell
across these three DVMM range settings. With DVMM at $80$ or $100$\,m, PROBE-X leads five of the
eight wide--narrow cells in Recall@1.

\begin{figure}[!tb]
  \centering
  \includegraphics[width=\columnwidth]{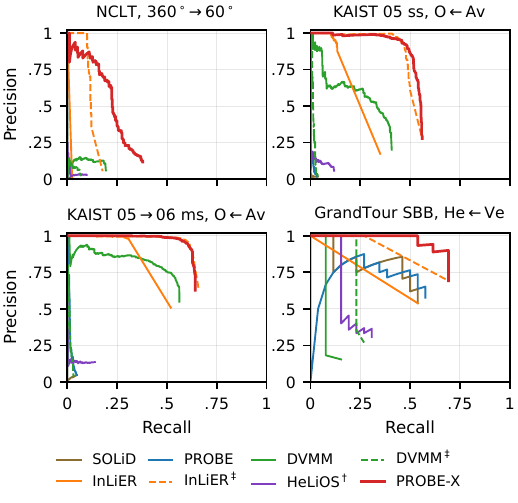}
  \caption{Precision--recall curves (top-1 fixed, score threshold swept)
  for the NCLT $360^\circ\!\to60^\circ$ cell of Table~\ref{tab:fov_sweep},
  the single- and multi-session Ouster$\leftarrow$Avia cells of
  Table~\ref{tab:cross}, and the SBB Hesai$\leftarrow$Velodyne cell of
  Table~\ref{tab:grandtour}. All methods from the respective tables except
  SC++ are plotted. Dashed curves use up to ten accumulated scans.}
  \label{fig:pr}
\end{figure}

\begin{figure}[!tb]
  \centering
  \includegraphics[width=\columnwidth]{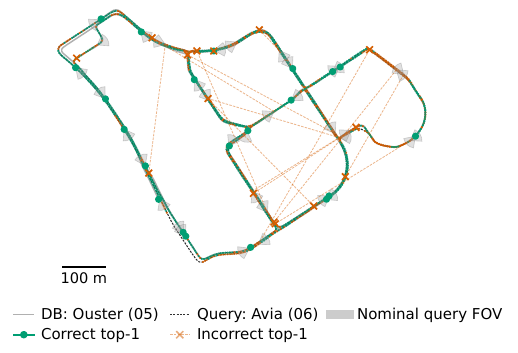}
  \caption{PROBE-X retrieval examples on KAIST $05\to06$ (Ouster
  database, Avia queries). Trajectories and wedge orientations use
  benchmark ground-truth poses. Small dots mark all
  $1{,}228$ queries with positives; colors indicate top-1 correctness
  under Sec.~\ref{sec:protocol}. Links and large markers show top-1
  matches for $36$ queries with positives spaced evenly in index order
  ($21$ correct). Wedges depict the nominal $70^\circ$
  query FOV with a $28$\,m display-only radius. Full-set Recall@1 is
  $0.642$.}
  \label{fig:examples}
\end{figure}

Fig.~\ref{fig:pr} compares precision--recall curves across the datasets.
When all verification scores are zero, InLiER retains retrieval order,
consistent with InLiER's Recall@1 evaluation.

Fig.~\ref{fig:examples} shows the spatial distribution of PROBE-X's
correct and incorrect top-1 matches for KAIST $05\to06$,
Ouster$\leftarrow$Avia.

\raggedbottom
On Riverside, PROBE-X leads five of twelve cross-sensor settings
in Recall@1, including three of four with panoramic databases (the
exception is multi-session Ouster$\leftarrow$Aeva). Mean Recall@1/AUC is
$0.37/0.30$ within session and $0.31/0.25$ across sessions.
The highest corresponding single-scan baseline means are
$0.30/0.21$ and $0.30/0.24$, both from InLiER.

\subsection{Evaluation on a Legged Platform}
\label{sec:grandtour}
Table~\ref{tab:grandtour} evaluates all nine pairings of ANYmal's three
panoramic LiDARs within GRI-1 and across SBB-1 (database) and SBB-2
(queries). This tests platform and sensor transfer rather than FOV
mismatch. Across eighteen cells,
single-scan PROBE-X has the highest mean Recall@1/AUC ($0.55/0.47$)
among the tabulated configurations, including InLiER$^\ddagger$
($0.51/0.39$) and DVMM$^\ddagger$ ($0.42/0.36$).
Among single-scan methods, PROBE-X leads Recall@1 outright in eight cells
and ties in one; PROBE leads four, including both same-sensor Hesai cells.
InLiER$^\ddagger$ exceeds PROBE-X in Recall@1 in five cells and ties in three.
On SBB, PROBE-X's Recall@1/AUC is $0.12/0.02$ for
Hesai$\leftarrow$Livox and $0.04/0.00$ for Velodyne$\leftarrow$Livox.

\begin{table*}[!t]
  \centering
  \caption{Cross-sensor place recognition on GrandTour with three panoramic LiDARs (Recall@1 / AUC). Rows: database; columns: query. Each setting has $20$--$26$ queries with positives (Sec.~\ref{sec:protocol}). $\dagger$: learned. $\ddagger$: up to ten accumulated scans; excluded from ranking (Sec.~\ref{sec:protocol}). Bold/underline indicate best/second-best single-scan values per metric, ranked at three-decimal precision.}
  \label{tab:grandtour}
  \footnotesize\setlength{\tabcolsep}{5pt}
  \resizebox{\textwidth}{!}{%
  \begin{tabular}{c c c c}
    \toprule
    DB$\backslash$Q & \textbf{Hesai} & \textbf{Livox} & \textbf{Velodyne} \\ \midrule
    \multicolumn{4}{l}{\textit{Single-session (GRI-1)}} \\ \midrule
    \textbf{Hesai} & \begin{tabular}[c]{@{}l@{}}SOLiD .57/.49 $\cdot$ PROBE \textbf{.76}/\textbf{.74} $\cdot$ DVMM .43/.40 \\ DVMM$^\ddagger$ .62/.53 $\cdot$ InLiER .62/.47 $\cdot$ InLiER$^\ddagger$ .71/.63 \\ HeLiOS$^\dagger$ .67/.49 $\cdot$ PROBE-X \underline{.71}/\underline{.71}\end{tabular} & \begin{tabular}[c]{@{}l@{}}SOLiD .23/.02 $\cdot$ PROBE .32/.16 $\cdot$ DVMM \underline{.55}/\underline{.44} \\ DVMM$^\ddagger$ .50/.39 $\cdot$ InLiER .41/.25 $\cdot$ InLiER$^\ddagger$ .36/.21 \\ HeLiOS$^\dagger$ .41/.13 $\cdot$ PROBE-X \textbf{.64}/\textbf{.47}\end{tabular} & \begin{tabular}[c]{@{}l@{}}SOLiD .59/.37 $\cdot$ PROBE \underline{.73}/\underline{.69} $\cdot$ DVMM .64/.52 \\ DVMM$^\ddagger$ .59/.51 $\cdot$ InLiER .45/.28 $\cdot$ InLiER$^\ddagger$ .64/.53 \\ HeLiOS$^\dagger$ .68/.49 $\cdot$ PROBE-X \textbf{.77}/\textbf{.74}\end{tabular} \\
    \arrayrulecolor{black!35}\cmidrule[0.3pt](lr){1-4}\arrayrulecolor{black}
    \textbf{Livox} & \begin{tabular}[c]{@{}l@{}}SOLiD \underline{.43}/.11 $\cdot$ PROBE \underline{.43}/\textbf{.30} $\cdot$ DVMM .33/.24 \\ DVMM$^\ddagger$ .43/.39 $\cdot$ InLiER \underline{.43}/\underline{.26} $\cdot$ InLiER$^\ddagger$ .48/.33 \\ HeLiOS$^\dagger$ \textbf{.48}/.14 $\cdot$ PROBE-X \underline{.43}/.25\end{tabular} & \begin{tabular}[c]{@{}l@{}}SOLiD \underline{.45}/.20 $\cdot$ PROBE .41/.12 $\cdot$ DVMM .36/.16 \\ DVMM$^\ddagger$ .41/.32 $\cdot$ InLiER \underline{.45}/\underline{.28} $\cdot$ InLiER$^\ddagger$ .45/.28 \\ HeLiOS$^\dagger$ .41/.28 $\cdot$ PROBE-X \textbf{.50}/\textbf{.40}\end{tabular} & \begin{tabular}[c]{@{}l@{}}SOLiD \textbf{.59}/.29 $\cdot$ PROBE .32/.19 $\cdot$ DVMM .36/\textbf{.33} \\ DVMM$^\ddagger$ .45/.40 $\cdot$ InLiER .41/.25 $\cdot$ InLiER$^\ddagger$ .50/.36 \\ HeLiOS$^\dagger$ \underline{.45}/.18 $\cdot$ PROBE-X .41/\underline{.30}\end{tabular} \\
    \arrayrulecolor{black!35}\cmidrule[0.3pt](lr){1-4}\arrayrulecolor{black}
    \textbf{Velodyne} & \begin{tabular}[c]{@{}l@{}}SOLiD .60/.50 $\cdot$ PROBE \textbf{.80}/\textbf{.78} $\cdot$ DVMM .60/.60 \\ DVMM$^\ddagger$ .55/.52 $\cdot$ InLiER .60/.42 $\cdot$ InLiER$^\ddagger$ .70/.56 \\ HeLiOS$^\dagger$ .55/.43 $\cdot$ PROBE-X \underline{.70}/\underline{.70}\end{tabular} & \begin{tabular}[c]{@{}l@{}}SOLiD .32/.04 $\cdot$ PROBE .55/.36 $\cdot$ DVMM \underline{.64}/\underline{.52} \\ DVMM$^\ddagger$ .50/.42 $\cdot$ InLiER .36/.21 $\cdot$ InLiER$^\ddagger$ .55/.35 \\ HeLiOS$^\dagger$ .36/.12 $\cdot$ PROBE-X \textbf{.73}/\textbf{.66}\end{tabular} & \begin{tabular}[c]{@{}l@{}}SOLiD .23/.03 $\cdot$ PROBE .41/.19 $\cdot$ DVMM .14/.00 \\ DVMM$^\ddagger$ .14/.02 $\cdot$ InLiER \underline{.45}/\underline{.28} $\cdot$ InLiER$^\ddagger$ .41/.25 \\ HeLiOS$^\dagger$ .41/.17 $\cdot$ PROBE-X \textbf{.59}/\textbf{.44}\end{tabular} \\
    \midrule
    \multicolumn{4}{l}{\textit{Multi-session (SBB-1 database, SBB-2 query)}} \\ \midrule
    \textbf{Hesai} & \begin{tabular}[c]{@{}l@{}}SOLiD .62/.60 $\cdot$ PROBE \textbf{.81}/\textbf{.80} $\cdot$ DVMM .46/.46 \\ DVMM$^\ddagger$ .54/.52 $\cdot$ InLiER .69/.63 $\cdot$ InLiER$^\ddagger$ .88/.87 \\ HeLiOS$^\dagger$ .54/.47 $\cdot$ PROBE-X \underline{.77}/\underline{.76}\end{tabular} & \begin{tabular}[c]{@{}l@{}}SOLiD .04/.00 $\cdot$ PROBE .04/.00 $\cdot$ DVMM \textbf{.27}/\underline{.08} \\ DVMM$^\ddagger$ .42/.38 $\cdot$ InLiER \underline{.23}/\textbf{.14} $\cdot$ InLiER$^\ddagger$ .31/.20 \\ HeLiOS$^\dagger$ .08/.00 $\cdot$ PROBE-X .12/.02\end{tabular} & \begin{tabular}[c]{@{}l@{}}SOLiD .54/\underline{.45} $\cdot$ PROBE \underline{.58}/.42 $\cdot$ DVMM .15/.09 \\ DVMM$^\ddagger$ .27/.24 $\cdot$ InLiER .54/.41 $\cdot$ InLiER$^\ddagger$ .69/.63 \\ HeLiOS$^\dagger$ .31/.21 $\cdot$ PROBE-X \textbf{.69}/\textbf{.68}\end{tabular} \\
    \arrayrulecolor{black!35}\cmidrule[0.3pt](lr){1-4}\arrayrulecolor{black}
    \textbf{Livox} & \begin{tabular}[c]{@{}l@{}}SOLiD .27/.09 $\cdot$ PROBE \textbf{.54}/\textbf{.31} $\cdot$ DVMM .15/.14 \\ DVMM$^\ddagger$ .50/.50 $\cdot$ InLiER .19/.11 $\cdot$ InLiER$^\ddagger$ .31/.20 \\ HeLiOS$^\dagger$ .35/.15 $\cdot$ PROBE-X \underline{.46}/\underline{.20}\end{tabular} & \begin{tabular}[c]{@{}l@{}}SOLiD .31/.13 $\cdot$ PROBE \underline{.46}/.24 $\cdot$ DVMM .19/.08 \\ DVMM$^\ddagger$ .42/.37 $\cdot$ InLiER .23/.14 $\cdot$ InLiER$^\ddagger$ .38/.27 \\ HeLiOS$^\dagger$ \textbf{.54}/\underline{.43} $\cdot$ PROBE-X \textbf{.54}/\textbf{.53}\end{tabular} & \begin{tabular}[c]{@{}l@{}}SOLiD .23/.04 $\cdot$ PROBE \underline{.27}/.07 $\cdot$ DVMM .19/.09 \\ DVMM$^\ddagger$ .15/.14 $\cdot$ InLiER .15/.09 $\cdot$ InLiER$^\ddagger$ .19/.11 \\ HeLiOS$^\dagger$ \textbf{.31}/\underline{.10} $\cdot$ PROBE-X .23/\textbf{.14}\end{tabular} \\
    \arrayrulecolor{black!35}\cmidrule[0.3pt](lr){1-4}\arrayrulecolor{black}
    \textbf{Velodyne} & \begin{tabular}[c]{@{}l@{}}SOLiD .46/.37 $\cdot$ PROBE \underline{.65}/\underline{.63} $\cdot$ DVMM .19/.15 \\ DVMM$^\ddagger$ .27/.26 $\cdot$ InLiER .62/.50 $\cdot$ InLiER$^\ddagger$ .58/.46 \\ HeLiOS$^\dagger$ .50/.41 $\cdot$ PROBE-X \textbf{.73}/\textbf{.65}\end{tabular} & \begin{tabular}[c]{@{}l@{}}SOLiD .00/.00 $\cdot$ PROBE .04/.00 $\cdot$ DVMM \underline{.15}/\underline{.06} \\ DVMM$^\ddagger$ .27/.16 $\cdot$ InLiER \textbf{.23}/\textbf{.14} $\cdot$ InLiER$^\ddagger$ .27/.17 \\ HeLiOS$^\dagger$ \textbf{.23}/.04 $\cdot$ PROBE-X .04/.00\end{tabular} & \begin{tabular}[c]{@{}l@{}}SOLiD .58/.45 $\cdot$ PROBE \underline{.73}/\underline{.72} $\cdot$ DVMM .50/.49 \\ DVMM$^\ddagger$ .50/.48 $\cdot$ InLiER .50/.38 $\cdot$ InLiER$^\ddagger$ .69/.65 \\ HeLiOS$^\dagger$ .62/.60 $\cdot$ PROBE-X \textbf{.85}/\textbf{.82}\end{tabular} \\
    \bottomrule
  \end{tabular}}
\end{table*}

\flushbottom

\subsection{Controlled Analysis of Retrieval Design}
\label{sec:controls}
Table~\ref{tab:controls} examines support conditioning and heading
search on NCLT $360^\circ\!\to60^\circ$ and KAIST $05\to06$
Ouster$\leftarrow$Avia. Row (i) uses the original PROBE pipeline.
Rows (i$'$) and (ii) use a single
forward hypothesis, full-ring alignment, and
PROBE-X's base scoring functions. Row (i$'$) disables support conditioning for the retrieval
keys, alignment input, and geometric support; row (ii) applies it.
The physical-FOV crop, overlap gate and weight, ground truth, and
exclusion rule are unchanged between these two rows. Row (ii$'$)
restricts the alignment in (ii) to $\pm48^\circ$. Row (iii) is PROBE-X,
with up to four hypotheses and $\pm48^\circ$ alignment windows.
Row (iii$'$) removes only retrieval-key conditioning from (iii),
preserving the hypothesis set, nomination and merging rules, and
alignment and scoring procedures. Both settings have panoramic databases, so query-key masks
already cover all bins; this ablation changes the database keys.

Rows (i$'$)--(iii$'$) nominate up to $K{=}20$ candidates per hypothesis.
Every row passes at most $K$ unique candidates to scoring without
using the ground-truth heading. Candidate coverage (Cov.) measures
whether a positive reaches the scored shortlist.

\begin{table}[!t]
  \centering
  \setlength{\abovecaptionskip}{2pt}
  \caption{Retrieval-design comparisons on common queries within each setting ($K{=}20$; Sec.~\ref{sec:controls}). KAIST: Ouster database, Avia queries. Cov.: fraction of queries with positives whose scored shortlist contains a positive.}
  \label{tab:controls}
  \footnotesize\setlength{\tabcolsep}{1.5pt}
  \setlength{\aboverulesep}{1pt}\setlength{\belowrulesep}{1pt}
  \begin{tabular}{@{}>{\raggedright\arraybackslash}m{16pt}@{\hspace{3pt}}>{\raggedright\arraybackslash}m{\dimexpr\columnwidth-162pt\relax}@{\hspace{5pt}}*{3}{>{\centering\arraybackslash}m{20pt}}@{\hspace{6pt}}*{3}{>{\centering\arraybackslash}m{20pt}}@{}}
    \toprule
    & & \multicolumn{3}{c}{NCLT $360^\circ\!\to60^\circ$} & \multicolumn{3}{c}{KAIST $05\to06$} \\
    \cmidrule(lr){3-5}\cmidrule(lr){6-8}
    \multicolumn{2}{@{}l}{Variant} & R@1 & AUC & Cov. & R@1 & AUC & Cov. \\ \midrule
    \parbox[c][18pt][c]{\linewidth}{\raggedright \strut (i)\par} & \parbox[c][18pt][c]{\linewidth}{\raggedright \strut PROBE\par} & \parbox[c][18pt][c]{\linewidth}{\centering \strut .013\par} & \parbox[c][18pt][c]{\linewidth}{\centering \strut .000\par} & \parbox[c][18pt][c]{\linewidth}{\centering \strut .03\par} & \parbox[c][18pt][c]{\linewidth}{\centering \strut .048\par} & \parbox[c][18pt][c]{\linewidth}{\centering \strut .015\par} & \parbox[c][18pt][c]{\linewidth}{\centering \strut .09\par} \\
    \parbox[c][18pt][c]{\linewidth}{\raggedright \strut (i$'$)\par} & \parbox[c][18pt][c]{\linewidth}{\raggedright \strut Unconditioned\par} & \parbox[c][18pt][c]{\linewidth}{\centering \strut .022\par} & \parbox[c][18pt][c]{\linewidth}{\centering \strut .000\par} & \parbox[c][18pt][c]{\linewidth}{\centering \strut .19\par} & \parbox[c][18pt][c]{\linewidth}{\centering \strut .086\par} & \parbox[c][18pt][c]{\linewidth}{\centering \strut .020\par} & \parbox[c][18pt][c]{\linewidth}{\centering \strut .23\par} \\
    \parbox[c][18pt][c]{\linewidth}{\raggedright \strut (ii)\par} & \parbox[c][18pt][c]{\linewidth}{\raggedright \strut Single hypothesis\\[-0.5pt]{\scriptsize\strut Masked, full-ring alignment\par}\par} & \parbox[c][18pt][c]{\linewidth}{\centering \strut .266\par} & \parbox[c][18pt][c]{\linewidth}{\centering \strut .206\par} & \parbox[c][18pt][c]{\linewidth}{\centering \strut .42\par} & \parbox[c][18pt][c]{\linewidth}{\centering \strut .671\par} & \parbox[c][18pt][c]{\linewidth}{\centering \strut .664\par} & \parbox[c][18pt][c]{\linewidth}{\centering \strut .70\par} \\
    \parbox[c][18pt][c]{\linewidth}{\raggedright \strut (ii$'$)\par} & \parbox[c][18pt][c]{\linewidth}{\raggedright \strut Single hypothesis\\[-0.5pt]{\scriptsize\strut Masked, $\pm48^\circ$ alignment\par}\par} & \parbox[c][18pt][c]{\linewidth}{\centering \strut .266\par} & \parbox[c][18pt][c]{\linewidth}{\centering \strut .206\par} & \parbox[c][18pt][c]{\linewidth}{\centering \strut .42\par} & \parbox[c][18pt][c]{\linewidth}{\centering \strut .671\par} & \parbox[c][18pt][c]{\linewidth}{\centering \strut .664\par} & \parbox[c][18pt][c]{\linewidth}{\centering \strut .70\par} \\
    \parbox[c][18pt][c]{\linewidth}{\raggedright \strut (iii)\par} & \parbox[c][18pt][c]{\linewidth}{\raggedright \strut \textbf{PROBE-X}\par} & \parbox[c][18pt][c]{\linewidth}{\centering \strut .379\par} & \parbox[c][18pt][c]{\linewidth}{\centering \strut .221\par} & \parbox[c][18pt][c]{\linewidth}{\centering \strut .53\par} & \parbox[c][18pt][c]{\linewidth}{\centering \strut .642\par} & \parbox[c][18pt][c]{\linewidth}{\centering \strut .633\par} & \parbox[c][18pt][c]{\linewidth}{\centering \strut .66\par} \\
    \parbox[c][18pt][c]{\linewidth}{\raggedright \strut (iii$'$)\par} & \parbox[c][18pt][c]{\linewidth}{\raggedright \strut PROBE-X\\[-0.5pt]{\scriptsize\strut Unconditioned retrieval keys\par}\par} & \parbox[c][18pt][c]{\linewidth}{\centering \strut .099\par} & \parbox[c][18pt][c]{\linewidth}{\centering \strut .030\par} & \parbox[c][18pt][c]{\linewidth}{\centering \strut .19\par} & \parbox[c][18pt][c]{\linewidth}{\centering \strut .217\par} & \parbox[c][18pt][c]{\linewidth}{\centering \strut .205\par} & \parbox[c][18pt][c]{\linewidth}{\centering \strut .23\par} \\
    \bottomrule
  \end{tabular}
\end{table}

From (i$'$) to (ii), joint conditioning of keys, alignment input, and
geometric support raises Recall@1 from $0.022$ to $0.266$ on NCLT and
from $0.086$ to $0.671$ on KAIST.
Disabling only retrieval-key conditioning in (iii$'$) reduces PROBE-X's
Recall@1 from $0.379$ to $0.099$ on NCLT and from $0.642$ to $0.217$
on KAIST. Candidate coverage drops from $0.53$ to $0.19$ and from
$0.66$ to $0.23$, respectively. In both settings, candidate coverage
without retrieval-key conditioning is lower than PROBE-X's achieved
Recall@1. Even perfect selection from those shortlists could not close
the observed recall gap.

Restricting the single-hypothesis alignment window, (ii)$\to$(ii$'$),
leaves Recall@1 and candidate coverage unchanged in both settings;
AUC is unchanged at the reported precision. With the same $\pm48^\circ$
window per hypothesis, adding heading hypotheses, (ii$'$)$\to$(iii),
increases NCLT Recall@1 from $0.266$ to $0.379$ but decreases KAIST
Recall@1 from $0.671$ to $0.642$.

To examine the KAIST decrease, we group queries by the relative headings
of all their positives: pure forward (all within $\pm48^\circ$, $249$
queries), pure reverse (all more than $135^\circ$ from the forward direction,
$191$), and mixed otherwise
($788$). Variant (ii) returns correct top-1 matches for $206$, $0$, and
$618$ queries in these groups, versus $154$, $73$, and $561$ for PROBE-X.
The reverse-direction
gain exceeds the pure-forward loss, while the mixed-group loss accounts
for the net decrease. On the KAIST narrow--narrow Aeva$\leftrightarrow$Avia cells, at
most $1.3\%$ of queries with positives have all of them outside the
retained $\pm48^\circ$ window.

\subsection{Runtime}
\label{sec:runtime}
\begin{table}[!t]
  \centering
  \caption{Per-query runtime of single-scan learning-free implementations in milliseconds (Sec.~\ref{sec:runtime}). Enc: encoding from clipped, in-memory scans; Retr: retrieval and method-specific scoring or verification. Values are medians except p95 (95th percentile of total time). Disk I/O, ground-truth evaluation, logging, and one-off database preparation are excluded.}
  \label{tab:runtime}
  \footnotesize
  \setlength{\tabcolsep}{3.2pt}
  \begin{tabular}{@{}lrrrr@{\hspace{9pt}}rrrr@{}}
    \toprule
    & \multicolumn{4}{c}{NCLT $360^\circ\!\to60^\circ$} & \multicolumn{4}{c}{KAIST $05\to06$} \\
    \cmidrule(lr){2-5}\cmidrule(lr){6-9}
    Method & Enc & Retr & Total & p95 & Enc & Retr & Total & p95 \\
    \midrule
    SOLiD~\cite{kim2024solid} & 0.6 & 1.6 & 2.3 & 3.8 & 1.0 & 3.1 & 4.1 & 4.6 \\
    PROBE~\cite{lee2026probe} & 3.9 & 1.7 & 5.6 & 8.2 & 7.4 & 1.7 & 9.1 & 11.3 \\
    DVMM~\cite{duan2025dvmm} & 0.9 & 24.5 & 25.6 & 28.5 & 2.8 & 28.7 & 32.0 & 35.7 \\
    InLiER~\cite{stathoulopoulos2026inlier} & 7.0 & 17.7 & 24.7 & 25.9 & 9.0 & 20.8 & 29.8 & 32.0 \\
    PROBE-X & 1.5 & 5.7 & 7.2 & 10.4 & 3.3 & 5.6 & 9.0 & 11.0 \\
    \bottomrule
  \end{tabular}
\end{table}

Table~\ref{tab:runtime} compares learning-free implementations on NCLT
$360^\circ\!\to60^\circ$ and KAIST $05\to06$ Ouster$\leftarrow$Avia.
HeLiOS is compared on accuracy only. Measurements use one process pinned
to one logical CPU of an AMD Ryzen 7 260, with numerical-library thread
limits set to one. After five warm-up queries, we time $200$ evenly
spaced queries over three repetitions and pool the timings.

PROBE-X's median query latency is $7.2$\,ms on NCLT and $9.0$\,ms on
KAIST. Compared with PROBE, its total latency is higher on NCLT and
comparable on KAIST. One-off
setup takes $8$ and $22$\,s for PROBE-X ($1{,}211$ and $1{,}294$ database
scans). InLiER and DVMM use C++ cores.

\section{Conclusion}
\label{sec:conclusion}
PROBE-X performs learning-free single-scan cross-FOV retrieval using
hypothesis-conditioned ring keys and geometric scoring on angular
overlap at refined headings. Single-scan PROBE-X achieves the highest
Recall@1 and AUC at all five asymmetric NCLT widths, including comparisons
with the evaluated accumulated-input baselines.
Among the evaluated single-scan methods, it also achieves the highest
AUC in all twelve KAIST cross-sensor settings. Ablations in two settings
show that conditioning retrieval-key construction improves candidate
coverage and that subsequent geometric scoring cannot compensate for the
candidate losses from unconditioned keys within the resulting shortlists.



\begin{thebibliography}{22}
\bibitem{yin2024global}
H.~Yin \etal, ``A survey on global LiDAR localization: Challenges, advances
and open problems,'' \textit{Int.\ J.\ Comput.\ Vis.}, vol.~132, no.~8,
pp.~3139--3171, 2024.
\bibitem{kim2018scan}
G.~Kim and A.~Kim, ``Scan Context: Egocentric spatial descriptor for place
recognition within 3D point cloud map,'' in \textit{Proc. IEEE/RSJ IROS},
2018, pp.~4802--4809.
\bibitem{kim2021scan}
G.~Kim, S.~Choi, and A.~Kim, ``Scan Context++: Structural place recognition
robust to rotation and lateral variations in urban environments,''
\textit{IEEE Trans. Robot.}, vol.~38, no.~3, pp.~1856--1874, 2022.
\bibitem{lee2026probe}
J.~Lee, B.~Lee, and G.~Yoo, ``PROBE: Probabilistic occupancy BEV encoding
with analytical translation robustness for 3D place recognition,''
\textit{IEEE Robot. Autom. Lett.}, vol.~11, no.~8, pp.~9779--9786, Aug.
2026, doi:~10.1109/LRA.2026.3703245.
\bibitem{duan2025dvmm}
X.~Duan, Q.~Hu, M.~Ai, P.~Zhao, and J.~Li, ``DVMM: A dual-view combination
descriptor for multi-modal LiDARs online place recognition,''
\textit{IEEE Robot. Autom. Lett.}, vol.~10, no.~10, pp.~10434--10441, 2025,
doi:~10.1109/LRA.2025.3600141.
\bibitem{stathoulopoulos2026inlier}
N.~Stathoulopoulos and G.~Nikolakopoulos, ``InLiER: Learning-free heterogeneous
LiDAR place recognition via intermediate mixed-radix structural keypoint
tokenization,'' \textit{IEEE Robot. Autom. Lett.}, vol.~11, no.~10,
pp.~11275--11282, 2026, doi:~10.1109/LRA.2026.3723737.
\bibitem{he2016m2dp}
L.~He, X.~Wang, and H.~Zhang, ``M2DP: A novel 3D point cloud descriptor and
its application in loop closure detection,'' in \textit{Proc.\ IEEE/RSJ
IROS}, 2016, pp.~231--237.
\bibitem{wang2020lidar}
Y.~Wang, Z.~Sun, C.-Z.~Xu, S.~E.~Sarma, J.~Yang, and H.~Kong,
``LiDAR Iris for loop-closure detection,'' in \textit{Proc.\ IEEE/RSJ
IROS}, 2020, pp.~5769--5775.
\bibitem{lu2023ring}
X.~Xu \etal, ``RING++: Roto-translation invariant Gram for global localization
on a sparse scan map,'' \textit{IEEE Trans. Robot.}, vol.~39, no.~6,
pp.~4616--4635, 2023.
\bibitem{yuan2023std}
C.~Yuan, J.~Lin, Z.~Zou, X.~Hong, and F.~Zhang, ``STD: Stable triangle
descriptor for 3D place recognition,'' in \textit{Proc.\ IEEE ICRA}, 2023,
pp.~1897--1903.
\bibitem{yuan2024btc}
C.~Yuan, J.~Lin, Z.~Liu, H.~Wei, X.~Hong, and F.~Zhang, ``BTC: A binary and
triangle combined descriptor for 3-D place recognition,'' \textit{IEEE
Trans.\ Robot.}, vol.~40, pp.~1580--1599, 2024.
\bibitem{gupta2026mapclosures}
S.~Gupta, T.~Guadagnino, B.~Mersch, N.~Trekel, M.~V.~R.~Malladi, and
C.~Stachniss, ``Efficiently closing loops in LiDAR-based SLAM using point
cloud density maps,'' \textit{Int.\ J.\ Robot.\ Res.}, 2026,
doi:~10.1177/02783649261449269.
\bibitem{chen2021overlapnet}
X.~Chen \etal, ``OverlapNet: Loop closing for LiDAR-based SLAM,'' in
\textit{Proc.\ RSS}, 2020.
\bibitem{jung2025helios}
M.~Jung, S.~Jung, H.~Gil, and A.~Kim, ``HeLiOS: Heterogeneous LiDAR place
recognition via overlap-based learning and local spherical transformer,''
in \textit{Proc. IEEE ICRA}, 2025, pp.~2204--2211.
\bibitem{luo2025bevplace}
L.~Luo \etal, ``BEVPlace++: Fast, robust, and lightweight LiDAR global
localization for autonomous ground vehicles,'' \textit{IEEE Trans. Robot.},
vol.~41, pp.~4479--4498, 2025.
\bibitem{shen2026unilgl}
H.~Shen \etal, ``UniLGL: Learning uniform place recognition for
FOV-limited/panoramic LiDAR global localization,'' \textit{IEEE Trans. Robot.},
vol.~42, pp.~1556--1576, 2026.
\bibitem{kim2024solid}
H.~Kim, J.~Choi, T.~Sim, G.~Kim, and Y.~Cho, ``Narrowing your FOV with SOLiD:
Spatially organized and lightweight global descriptor for FOV-constrained
LiDAR place recognition,'' \textit{IEEE Robot. Autom. Lett.}, vol.~9,
no.~11, pp.~9645--9652, 2024.
\bibitem{carlevaris2016university}
N.~Carlevaris-Bianco, A.~K.~Ushani, and R.~M.~Eustice, ``University of
Michigan North Campus long-term vision and lidar dataset,'' \textit{Int.\
J.\ Robot.\ Res.}, vol.~35, no.~9, pp.~1023--1035, 2016.
\bibitem{minwoo2023helipr}
M.~Jung, W.~Yang, D.~Lee, H.~Gil, G.~Kim, and A.~Kim, ``HeLiPR:
Heterogeneous LiDAR dataset for inter-LiDAR place recognition under
spatiotemporal variations,'' \textit{Int.\ J.\ Robot.\ Res.}, vol.~43,
no.~12, pp.~1867--1883, 2024.
\bibitem{frey2026grandtour}
T.~Tuna \etal, ``GrandTour: A legged robotics dataset in the wild for
multi-modal perception and state estimation,'' arXiv:2602.18164, 2026.
\end{thebibliography}
\end{document}